\documentclass[11pt]{article}
\usepackage[utf8]{inputenc}

\usepackage[margin=1.0in]{geometry}
\usepackage[english]{babel}
\usepackage{microtype}
\usepackage{verbatim}
\usepackage{authblk}

\usepackage{amsmath, amsthm, amsfonts, graphicx, bbm}
\usepackage{hyperref}
\usepackage{subcaption, wrapfig}

\usepackage{multirow, array}

\newcommand{\data}{\mathcal{D}}

\newcommand{\myparagraph}[1]{ \noindent \textbf{#1.}}

\DeclareMathOperator*{\argmax}{arg\,max}
\DeclareMathOperator*{\argmin}{arg\,min}

\newcommand{\filtfunc}{\mathcal{F}}

\usepackage[noend]{algpseudocode}
\usepackage{algorithm}
\def\comments{0}

\ifnum\comments=1
    \newcommand{\matthew}[1]{\marginpar{\tiny\color{orange}{MJ: #1}}}
    \newcommand{\giorgio}[1]{\marginpar{\tiny\color{blue}{GS: #1}}}
    \newcommand{\alina}[1]{\marginpar{\tiny\color{purple}{AO: #1}}}
\else
    \newcommand{\matthew}[1]{}
    \newcommand{\giorgio}[1]{}
    \newcommand{\alina}[1]{}
\fi

\newcommand{\ignore}[1]{}
\theoremstyle{definition}
\newtheorem{definition}{Definition}[section]
\newtheorem{theorem}{Theorem}[section]

\newcommand{\fmatch}{\ensuremath{\textsc{FeatureMatch}}}
\newcommand{\cmatch}{\ensuremath{\textsc{ClusterMatch}}}
\newcommand{\sbert}{\ensuremath{\textsc{BERT-LL}}}
\newcommand{\lbert}{\ensuremath{\textsc{BERT-FT}}}
\newcommand{\scifar}{\ensuremath{\textsc{Conv}}}
\newcommand{\lcifar}{\ensuremath{\textsc{VGG-FT}}}
\newcommand{\sutk}{\ensuremath{\textsc{VGG-LL}}}
\newcommand{\lutk}{\ensuremath{\textsc{VGG-FT}}}

\title{Subpopulation Data Poisoning Attacks}

\author{Matthew Jagielski}
\author{Giorgio Severi}
\author{Niklas Pousette Harger}
\author{Alina Oprea}
\affil{Khoury College of Computer Sciences, Northeastern University}
\date{}

\begin{document}
\maketitle

\begin{abstract}

Machine learning  systems are deployed in critical settings,  but they might fail in unexpected ways, impacting the accuracy of their predictions. Poisoning attacks against machine learning induce adversarial modification of data used by a machine learning algorithm  to  selectively  change  its output   when  it is deployed. In this work, we introduce a novel data poisoning attack called a \emph{subpopulation attack}, which is particularly relevant when datasets are large and diverse. We design a modular framework for subpopulation attacks, instantiate it with different building blocks, and show that the attacks are effective  for a variety of datasets and machine learning models. We further optimize the attacks in continuous domains using influence functions and gradient optimization methods. Compared to existing backdoor poisoning attacks, subpopulation attacks have the advantage of inducing misclassification in naturally distributed data points at inference time, making the attacks extremely stealthy. We also show that our attack strategy can be used to improve upon existing targeted attacks. We prove that, under some assumptions, subpopulation attacks are impossible to defend against, and empirically demonstrate the limitations of existing defenses against our attacks, highlighting the difficulty of protecting machine learning against this threat.
\end{abstract}

\section{Introduction}

Machine learning (ML) and deep learning systems are being deployed in sensitive applications, but they can fail in multiple ways, impacting the confidentiality, integrity and availability of user data~\cite{kumar2019failure}. To date, evasion attacks or inference-time attacks have been studied extensively in image classification~\cite{Szegedy14,Goodfellow14,Carlini17}, speech recognition~\cite{carlini2018audio,schonherr2018adversarial}, and cyber security~\cite{Srndic14,xu2016automatically,kreuk2018deceiving}. Still, among the threats machine learning and deep learning systems are vulnerable to, poisoning attacks at training time has recently surfaced as the threat perceived as most potentially dangerous to  companies' ML infrastructures~\cite{kumar2020adversarial}. The threat of poisoning attacks becomes even more severe as modern deep learning systems rely on large, diverse datasets, and their size makes it difficult to guarantee the trustworthiness of the training data.


In existing poisoning attacks, adversaries can insert a set of corrupted, poisoned data at training time to induce a specific outcome in classification at inference time.
Existing poisoning attacks can be classified into: \emph{availability attacks}~\cite{Biggio2012PoisoningAA,xiao2015feature,jagielski2018manipulating} in which  the overall accuracy of the model is degraded; \emph{targeted attacks}~\cite{koh2017understanding, shafahi2018poison, suciu2018does} in which specific test instances are targeted for misclassification; and
\emph{backdoor attacks}~\cite{gu2017badnets} in which a backdoor pattern added to testing points induces misclassification. Poisoning attacks range in the amount of knowledge the attacker has about the ML system, varying for example in knowledge about feature representation, model architecture, and training data~\cite{suciu2018does}.

The threat models for poisoning attacks defined in the literature rely on  strong assumptions on the adversarial capabilities. In both poisoning availability attacks~\cite{Biggio2012PoisoningAA, xiao2015feature,jagielski2018manipulating} and backdoor attacks~\cite{gu2017badnets} the adversary needs to control a relatively large fraction of the training data (e.g., 10\% or 20\%) to influence the model at inference time.  Moreover, in backdoor attacks an adversary is assumed to have the capability of modifying both the training and the testing data to include the backdoor pattern. In targeted attacks, there is an assumption that the adversary has knowledge on the exact target points during training, which is not always realistic. Additionally, the impact of a targeted attack is localized to a single point or a small set of points~\cite{shafahi2018poison,geiping2020witches}. In our work, we attempt to bridge gaps in the literature on poisoning attacks.




\subsection{Our Contributions}

\myparagraph{Subpopulation poisoning attacks}
We introduce a novel, realistic, form of data poisoning attack we call the \emph{subpopulation attack}, which is particularly relevant for large, diverse datasets. In this attack, an adversary's goal is to compromise the performance of a classifier on a particular subpopulation of interest, while maintaining unaltered its performance for the rest of the data. The advantages of our novel subpopulation attack are that it requires no adversarial knowledge of the exact model architecture and parameters,
and most importantly, the attack does not need to modify points at inference time, which has been a common thread in prior backdoor poisoning attacks~\cite{gu2017badnets,Turner2018CleanLabelBA}. We also uncover an interesting connection with research in algorithmic fairness, which shows that ML classifiers might act differently for different minority groups~\cite{hardt2016equality,buolamwini2018gender}.
We believe this fairness disparity contributes to ML classifiers' vulnerability to our stealthy subpopulation attacks. 

\myparagraph{Subpopulation attack framework}
We propose a  modular framework for conducting subpopulation attacks that includes a subpopulation identification component and an attack generation procedure. We instantiate these components with different building blocks, demonstrating the modularity of the design. Subpopulations can be identified either by exact feature matching when annotations are available or clustering in representation layers of the neural network models. For poisoning attack generation we evaluate several methods: a generic label flipping attack applicable to data from multiple modalities and an attack based on influence functions~\cite{koh2017understanding}. We also propose a gradient optimization attack that improves upon label flipping on continuous data, and has better performance than the influence function attack.

\myparagraph{Evaluation on end-to-end and transfer learning} We demonstrate the effectiveness of our framework on datasets from multiple modalities, including tabular data (UCI Adult), image data (CIFAR-10 for image classification, UTKFace for face recognition), and text data (IMDB for sentiment analysis). We use a range of neural network models, including feed-forward neural networks, convolutional networks, and transformers. We show that both end-to-end trained models, as well as transfer learning models are vulnerable to this new attack vector. Additionally, the size of the attack is small relative to the overall dataset and the poisoned points follow the training data distribution, making the attacks extremely stealthy. For instance, with only 126 points, we induce a classification error of 74\% on a subpopulation in  CIFAR-10, maintaining similar accuracy to the clean model on points outside the subpopulation. In the UTKFace face recognition dataset, an attack of 29 points increases classification error by 50\% in one subpopulation. We also show that generating subpopulations can improve targeted attacks~\cite{koh2017understanding, geiping2020witches}; for example, we show that the Witches' Brew attack~\cite{geiping2020witches} is 86\% more effective when targeting points from a single subpopulation.

\myparagraph{Challenges for defenses}
Several defenses against availability attacks~\cite{jagielski2018manipulating,diakonikolas2018sever}, targeted attacks~\cite{Peri2019DeepKD}, and backdoor attacks~\cite{chen_detecting_2018,liu_fine-pruning_2018, tran_spectral_2018, wang_neural_2019,veldanda_nnoculation_2020} have been proposed. We believe their stealthiness makes subpopulation attacks difficult to defend against. To support this claim, we provide an impossibility result, showing that models based on local decisions, such as mixture models and $k$-nearest neighbors, are inherently vulnerable to subpopulation attacks. We also evaluate the performance of existing defenses such as TRIM/ILTM~\cite{jagielski2018manipulating, shen2019learning}, SEVER~\cite{diakonikolas2018sever}, activation clustering~\cite{chen_detecting_2018}, and spectral signatures~\cite{tran_spectral_2018} on our attacks and show their limitations.

\section{Background}

\subsection{Neural Networks Background}
Consider a training set of $n$ examples $D=\lbrace x_i, y_i\rbrace_{i=1}^n$, with each feature vector $x_i\in\mathcal{X}$ and label $y_i\in \mathcal{Y}$ drawn from some data distribution $\mathcal{D}$. In this paper, we consider multiclass classification tasks, where a $K$-class problem has $\mathcal{Y}=[K]$. The goal of a learning algorithm $A$, when given a dataset $D$ is to return a model $f$ with parameters $\theta$ which correctly classifies as many data points as possible, maximizing $\mathbb{E}_{x,y\sim\mathcal{D}}\mathbbm{1}\left(f(x)=y\right)$. 
The way in which this is typically done is by stochastic gradient descent on a differentiable loss function. To approximate minimizing error, the model is designed to output probabilities of each of the $K$ classes (we will write the $i$th class probability as $f(x)_i$). In multiclass classification, it is typical to minimize the categorical cross-entropy loss $\ell$:
$$
\ell(X, Y, f) = \frac{1}{|X|}\sum_{i=1}^{|X|}y_i (1 - f(x_i)_{y_i}).
$$

In stochastic gradient descent the loss function is decreased by updating $\theta$ in the steepest descent direction, the gradient
$\nabla_{\theta}\ell(X, Y, f)$. Because computing this entire gradient is a slow and memory-intensive operation, stochastic gradient descent uses a random batch of data $X_B, Y_B$ to compute a gradient in an iteration, using the following update strategy:
$$
\theta_{new} = \theta_{old} - \eta \nabla_\theta \ell(X_B, Y_B, \theta),
$$
where $\eta$ is called the learning rate and $\theta_{new}$ replaces $\theta_{old}$ in the next iteration. This process is repeated for a specific number of iterations (a full pass over the training set is called an \emph{epoch}), or until a convergence criterion is met. There exist a large number of alternative optimization algorithms, such as momentum and Adam~\cite{kingma2014adam}, which leverage the history of updates to more efficiently minimize the loss function.

In a neural network, the model $f$ is a chain of linear and nonlinear transformations. The linear transformations are called layers, and the nonlinear transformations are activation functions. Examples activation functions are   ReLU, sigmoid, and softmax. 
Domain-specific neural network architectures have been successful - convolutional neural networks~\cite{lecun2015lenet} are widely used for images, and recurrent neural networks~\cite{hochreiter1997long} and transformers~\cite{vaswani2017attention} are popular for text.

In standard training, the model parameters $\theta$ are randomly initialized. A widespread alternative approach is called transfer learning~\cite{pan2009survey}, in which knowledge is transferred from a large dataset. In this approach, a model trained on a large dataset is used as an initialization for the smaller dataset's model. In this case, a common approach is to simply use the pretrained model as a feature extractor, keeping all but the last layer fixed, and only training the last layer. Another approach is to use it as an initialization for standard training, and allowing the whole network to be updated. These are both common approaches~\cite{kornblith2019better}, which we refer to as "last-layer" transfer learning and "fine-tuning" transfer learning, respectively. 

\subsection{Poisoning Attacks}
\begin{table*}[]
    \renewcommand\arraystretch{1.15}
    \tabcolsep=.15cm
    \centering
    \small
    \begin{tabular}{c | c c c c}
        \hline
        Attack & Data modality & No training data & Natural Target & Generalizes \\
        \hline
        Targeted Poisoning & Image & \checkmark & \checkmark & $\times$ \\
        Federated Targeted & Image/Text & \checkmark & \checkmark & $\times$ \\
        Reflection Backdoor & Image & \checkmark & $\times$ & \checkmark \\
        Composite Backdoor & Image/Text & $\times$ & $\times$ & \checkmark \\
        \textbf{Subpopulation (Ours)} & \textbf{Image/Text/Tabular} & \textbf{\checkmark} & \textbf{\checkmark} & \textbf{\checkmark} \\
    \end{tabular}
    \caption{Comparison to related work. The table shows the data modality the attack applies to, whether the attack requires the knowledge of the exact training points used by the victim model, whether the attacker needs to be able to modify points during inference (natural target), and if the attack generalizes to points which are not known at the time of attack. Targeted poisoning attacks include \cite{geiping2020witches, koh2017understanding}. Federated backdoor is from \cite{bagdasaryan2018backdoor}. Reflection backdoor is from ~\cite{liu2020reflection}. Composite backdoor is from ~\cite{lin2020composite}.}
    \label{tab:assumptions}
    \vspace{-10pt}
\end{table*}

In settings with large training sets, machine learning is vulnerable to \emph{poisoning attacks}, where an adversary is capable of adding corrupted data into the training set. This is typically because data is collected from a large number of sources which cannot all be trusted. For example, OpenAI trained their GPT-2 model on all webpages where at least three users of the social media site Reddit interacted with the link \cite{radford2019language}.
Google also trains word prediction models from data collected from Android phones~\cite{hard2018federated}. An adversary with a small number of Reddit accounts or Android devices can easily inject data into these models' training sets. Furthermore, a recent survey on companies' perception of machine learning threats \cite{kumar2020adversarial} also highlighted poisoning attacks as one of the most concerning attacks.

More formally, in a poisoning attack, the adversary adds $m$ \emph{contaminants} or poisoning points $D_p=\lbrace x_i^p, y_i^p \rbrace_{i=1}^m$ to the training set, so that the learner minimizes the poisoned objective $\ell(D \cup D_p, f)$ rather than $\ell(D, f)$. The poisoned set $D_p$ is constructed to achieve some adversarial objective $\mathcal{L}(A(D\cup D_p))$.

Prior objectives for poisoning attacks distinguish between a targeted distribution to compromise performance on, $\data_{targ}$, and a distribution to maintain the original classifier's performance on, $\data_{clean}$. Then attacks can be measured in terms of two metrics, the \emph{collateral} damage and the \emph{target} damage. A collateral damage constraint requires the accuracy of the classifier on $\data_{clean}$ to be unaffected, while the target damage requires the performance on $\data_{targ}$ to be compromised.
Indeed, existing poisoning attacks can be  grouped into the following taxonomy:

\myparagraph{Availability Attacks \cite{Biggio2012PoisoningAA, mei2015using, xiao2015feature, jagielski2018manipulating}} In an availability attack, the adversary wishes to arbitrarily reduce the model's classification performance. The target distribution is the original data distribution, with no collateral distribution.

\myparagraph{Targeted Attacks \cite{koh2017understanding, suciu2018does, shafahi2018poison}} In a targeted attack, the adversary has a small set of targets $\lbrace x_i^t,y_i^t\rbrace$ they seek to misclassify. The target distribution $\data_{targ}$ is $\lbrace x_i^t,y_i^t\rbrace$, while the collateral distribution removes $\lbrace x_i^t,y_i^t\rbrace$ from the support: $\data_{clean} = \data \setminus \lbrace x_i^t,y_i^t\rbrace$.

\myparagraph{Backdoor Attacks \cite{chen2017targeted, gu2017badnets}}: A backdoor attack is one in which the adversary is able to modify one or a few features of their input, seeking to cause predictable misclassification given the control of these few features. Then the collateral distribution is the natural data distribution $\data_{clean}=\data$, while the target distribution shifts the original data distribution by adding the backdoor perturbation, making $\data_{targ}=\textsc{Pert}(\data)$. Here $\textsc{Pert}$ is the function that adds a backdoor pattern to a data point.


\subsection{Related Work}

Adversarial attacks against machine learning  systems can be classified into \emph{poisoning attacks} at training time~\cite{Biggio2012PoisoningAA}, \emph{evasion attacks} at testing time~\cite{biggio2013evasion, Szegedy14}, and \emph{privacy attacks} to extract private information by interacting with the machine learning system~\cite{shokri2017membership, yeom2018privacy}. Below, we survey in more depth the different types of poisoning attacks and defenses in the literature.   

\vspace{-0.1cm}
\paragraph{Poisoning Availability Attacks.} The idea of tampering with the training data of an automated classifier to introduce failures in the final model has been the focus of many research efforts over time. Some of the early works in this area include attacks against polymorphic worms detectors~\cite{perdisci_misleading_2006}, network packet anomaly detectors~\cite{rubinstein2009antidote}, and behavioral malware clustering~\cite{biggio2014poisoning}. Availability attacks based on gradient descent have been proposed for multiple models, such as linear regression~\cite{xiao2015feature,jagielski2018manipulating}, logistic regression~\cite{mei15teaching}, and SVM~\cite{Biggio2012PoisoningAA}. These attacks have the goal to indiscriminately compromise the accuracy of the model. Regarding defenses, SEVER~\cite{diakonikolas2018sever} uses SVD to remove points which bias gradients, while TRIM~\cite{jagielski2018manipulating} and ILTM~\cite{shen2019learning} remove points with high loss. 
These defenses work iteratively, by identifying and removing outlying points at each step, until convergence is reached.
Demontis et al.~\cite{demontis19transfer} study the transferability of poisoning availability attacks. 

\vspace{-0.1cm}
\paragraph{Backdoor Attacks.} While \emph{red-herring} attacks~\cite{newsome_paragraph_2006} can be considered as precursors to \emph{backdoor} attacks, Gu et al.~\cite{gu2017badnets} is generally regarded as the first backdoor attack against modern neural networks. It identified a security issue with ML-as-a-service models, and involved generating poisoned data with a backdoor pattern to influence the model to classify incorrectly new backdoored testing points.  Successive work introduce \emph{clean-label} backdoor attacks which assume that the adversary does not control the labeling function~\cite{Turner2018CleanLabelBA}. Other applications of machine learning, such as Federated Learning models, have been shown to be vulnerable to backdoor attacks~\cite{bagdasaryan2018backdoor}. 
To defend against backdoor attacks, \cite{tran_spectral_2018} use SVD decomposition on the latent space learned by the network to develop an outlier score. \cite{liu_fine-pruning_2018} combines pruning and fine-tuning the network. \cite{wang_neural_2019} identify poisoning by measuring the minimum perturbation necessary to transform inputs into a target class.

\paragraph{Targeted Attacks.} Shafahi et al.~\cite{shafahi2018poison} introduce a clean-label, opti\-mization-based targeted poisoning attack. Suciu et al.~\cite{suciu2018does} study the transferability of targeted attacks. Schuster et al.~\cite{schuster2020humpty} show targeted poisoning attacks on  word embedding models used for NLP tasks.
Koh et al.~\cite{koh2017understanding} introduce an influence-based targeted attack and show an example on targeted attacks affecting multiple points at once. Koh et al.~\cite{koh2019accuracy} evaluate the effectiveness of influence functions for predicting how models change when multiple points are removed from a dataset. Our work uses larger datasets and models, and constructs attacks, which add poisoned points to influence predictions on the target subpopulation. Witches' Brew~\cite{geiping2020witches} uses a gradient matching loss, along with various optimization techniques, to perform targeted attacks which require little information of the learner's setup. We show in Section~\ref{sec:targeted} that our identification of subpopulations can be used to make targeted attacks more efficient, comparing to both \cite{koh2017understanding} and \cite{geiping2020witches}.


\paragraph{Related Attacks - Table~\ref{tab:assumptions}} We outline in Table~\ref{tab:assumptions} a comparison with other attacks in prior work that are closer to the attacks we present. We include targeted attacks, such as \cite{bagdasaryan2018backdoor, geiping2020witches}, which are capable of attacking multiple points. These types of attacks do not generalize to new points, however--- they are only capable of harming the specific points the attack was generated for. The reflection backdoor~\cite{liu2020reflection} adds natural image modifications induced by light reflection. The composite backdoor~\cite{lin2020composite} adds natural modifications, such as certain faces, into a sample to induce misclassifications. These types of backdoors, while still using natural images, require test-time modifications to induce misclassification. Our attacks are the only which do not require test time misclassifications, generalize to unseen points, function in a variety of data modalities, and do not require knowledge of the exact training set. Other related attacks exist for specific settings. This includes Neural Trojans \cite{liu2017neural}, which present an attack on the MNIST dataset which controls the classification on other digits. Kulynych et al.~\cite{kulynych2020pots} propose ``protective optimization technologies'', where users may craft data to improve models' performance on specific groups; our work focuses on malicious data, where there are additional concerns, such as minimizing collateral damage, understanding defenses, and designing the target subpopulation. In concurrent work, Nolans et al.~\cite{solans2020poisoning} and Chang et al.~\cite{chang2020adversarial} demonstrate attacks on models specifically designed to satisfy algorithmic fairness restrictions, a consequence of our attacks we discuss in Section~\ref{sec:fairness}.
\section{Subpopulation Attacks}
\label{sec:algorithm}

In this section, we introduce the threat model, main definitions, and the principal features that characterize our newly proposed subpopulation attack.

\subsection{Threat Model}


Similar to most forms of poisoning attack, the goal of the adversary is to introduce a small quantity of contaminants in the data used to train a machine learning classification model in order to introduce a desired characteristic in the learned parameters.
We consider a realistic adversary who does not have access to the internal parameters of the victim model, and, similarly to availability poisoning attacks, cannot modify any data point submitted to the victim model at testing time.  Moreover, the adversary is unable to gain knowledge of the exact data points employed for training, and can modify the training set only by adding new points. This reflects the scenario in which the attacker can only disseminate the poisoned points, which will then be gathered, together with a large quantity of benign points, by the developers of the victim model to create the training set.
However, we allow the adversary to have the computational power required to train a model comparable to the victim one, and to have access to a separate \emph{auxiliary} dataset $D_{aux}$, distinct from the training data $D$, sampled from the same distribution. We also allow the adversary knowledge of the learner's loss function (we consider only the widely used cross entropy loss) and architecture (an assumption we remove in Section~\ref{sec:transfer}). The adversary has no knowledge of the victim model parameters, or the actual training data. 

Many poisoning attacks in the literature employ a white-box attack model, with full knowledge of the training set, model architecture and parameters~\cite{Biggio2012PoisoningAA,xiao2015feature,jagielski2018manipulating, koh2017understanding}, and many transfer the attack from another model~\cite{demontis19transfer, geiping2020witches}. 
Here, we make the assumption that the adversary is able to access an auxiliary dataset $D_{aux}$, which we believe is reasonable given the availability of public datasets, and has been a common assumption for black-box attacks in prior work~\cite{jagielski2018manipulating, papernot2017practical}. 
While this assumption could potentially be removed with a good generative model, we leave the exploration of this hypothesis to future work.  Finally, we consider stealthiness and practicality to be highly important for the attacker, and therefore, we assume the adversary will be able to poison only a small number of points in the training set.

\subsection{Definition}\label{subsec:definition}
We propose a new type of subpopulation attack, that is an interpolation in the space of poisoning attacks between a targeted attack (misclassifying a small fixed set of points) and an availability attack (misclassifying as many points as possible). To define our attack, we first provide a general, intuitive definition of a subpopulation:

\begin{definition}
A subpopulation of a data distribution is a restriction of its input domain to a set of points which are close to each other based on some distance function. 
\end{definition}

This definition can capture multiple types of subpopulations. 
Individual fairness definitions from the algorithmic fairness literature correspond to subpopulations defined with $\ell_2$ distance, but may be unnatural for image data.
As we will see, subpopulations defined by $\ell_2$ distance in a trained model's representation space better capture image similarity. In algorithmic fairness, group fairness uses subpopulations defined by a distance function which is 0 for members of the same subpopulation and 1 for members of different subpopulations.

For a targeted subpopulation, the adversary's goal is twofold---impact the predictions on inputs coming from the subpopulation in the data, but do not impact the performance of the model on points outside this subpopulation. Crucially, this subpopulation consists of natural data, and does not require modifying points to observe the attack, as is the case for backdoor attacks. We allow the adversary to pick a subpopulation by selecting a \emph{filter function}, which partitions the population into the subpopulation to impact and the remainder of the data, whose performance should not change. Formally:
\begin{definition}{\emph{Subpopulation Attacks.}} Fix some learning algorithm $A$ and training dataset $D$ (not necessarily known to the adversary). A subpopulation attack consists of a dataset of contaminants $D_p$ and a \emph{filter function} $\mathcal{F}:\mathcal{X}\rightarrow\lbrace 0, 1\rbrace$ for defining subpopulations. $D_p$ is the poisoning set constructed to minimize the collateral damage and maximize the target damage on the subpopulation of interest when appended to the training set:
\begin{equation}
\begin{split}
    \textsc{Collat}(\mathcal{F}, D_p) = \mathbb{E}_{(x,y)\sim \mathcal{D}}[\mathbbm{1}\left(A(D\cup D_p)(x) \neq y\right) - \mathbbm{1}\left(A(D)(x) \neq y\right)~|~\mathcal{F}(x)=0]
\end{split}
\label{eq:maxclean}
\end{equation}

\begin{equation}
\begin{split}
    \textsc{Target}(\mathcal{F}, D_p) = \mathbb{E}_{(x,y)\sim \mathcal{D}}[\mathbbm{1}\left(A(D\cup D_p)(x) \neq y\right) - \mathbbm{1}\left(A(D)(x) \neq y\right)~|~\mathcal{F}(x)=1]
\end{split}
\label{eq:mintarget}
\end{equation}
\end{definition}

We will evaluate subpopulation attacks by reporting the collateral damage (\ref{eq:maxclean}) and the target damage (\ref{eq:mintarget}) on the unseen test set $D_{test}$. Crucially, we use $D_{test}$ to ensure that the attack generalizes to new points in the subpopulation, contrasting with prior attacks which only target a pre-specified set of points. A successful attack will have small collateral damage and large target damage, using few poisoning points.
Note that, under our definition, target damage simply seeks to maximize the classification error on the subpopulation, making it a \emph{class-untargeted} attack~\cite{carlini2017towards}; the definition can be easily modified to capture a \emph{class-targeted} attack, where samples from the subpopulation satisfying the filter function should be classified into a specific target class.

Subpopulation attacks are an interpolation between targeted poisoning attacks and availability poisoning attacks. The extreme case in which the filter selects a single point (or small set of points) corresponds to targeted poisoning attacks~\cite{shafahi2018poison}. When the filter function is defined to select the entire data domain, the attack is an availability attack~\cite{jagielski2018manipulating}. However, the most interesting subpopulation attacks, as we will demonstrate, use a relatively small number of poisoning points to attack a subpopulation of interest, while minimizing the collateral on other data. These subpopulation attacks are stealthy and hard to detect, in comparison with availability attacks, and have a potentially larger impact than a targeted attack.
The choice of filter function is as important to the adversary as the selection of contaminants. There may be some choices of filter function which result in subpopulations  harder to attack. For instance, a subpopulation closer to the decision boundary, on which the model has low confidence, may be attacked easier than a high-confidence subpopulation. 
In the next section, we will discuss our framework for generating these subpopulation attacks -  decomposed into subpopulation selection and attack generation.

\section{Methodology}
\label{sec:method}
\begin{figure*}[t]
    \centering
    \includegraphics[width=7in]{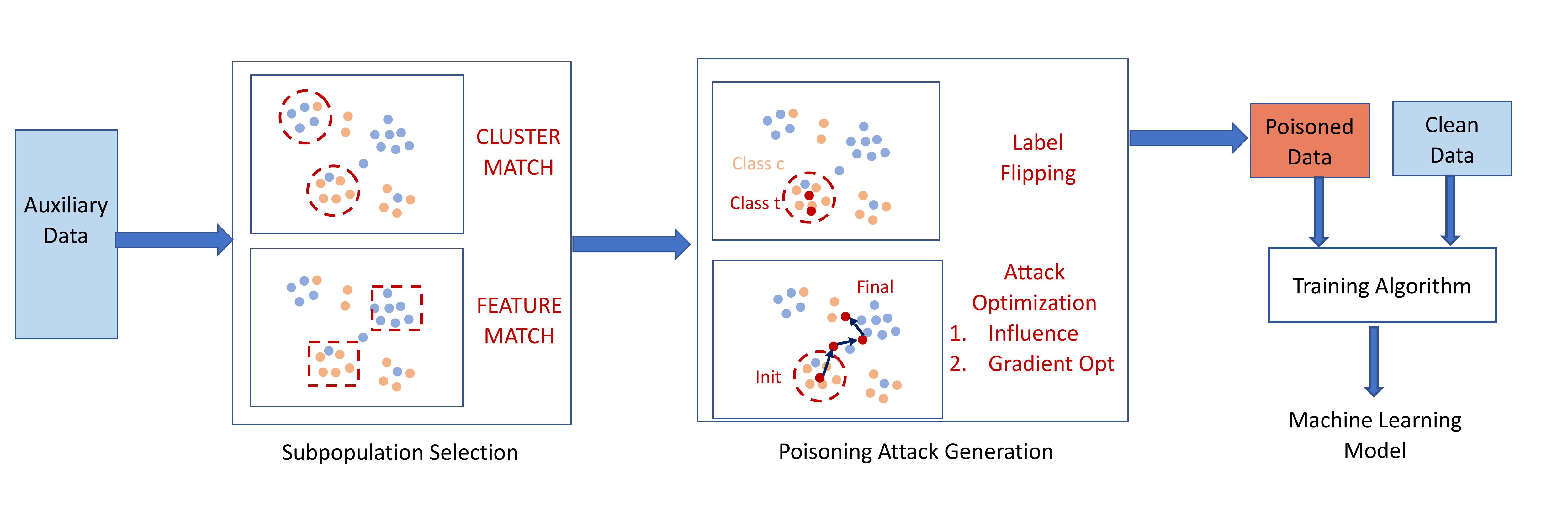}
    \caption{Overview of our subpopulation attack framework. The attacker has access to an auxiliary dataset, from which it can determine vulnerable subpopulation by using either \fmatch\ or \cmatch. Poisoning attack generation can be done by  label flipping (where a point drawn from a subpopulation with majority class $c$ is added with  label $t \neq c$), or with attack optimization (starting from label flipping, use either influence or gradient optimization for the final attack point).}
    \label{fig:overview}
\end{figure*}

There are two main components that contribute to a full-fledged subpopulation attack, as shown in Figure~\ref{fig:overview}: choosing a suitable filter function to select subpopulations, and selecting an algorithm to generate the poisoning data.
In this section we introduce our methodology  to obtain a complete subpopulation attack and describe the components in our framework.

\subsection{Subpopulation Selection}

The adversary first needs to identify their target subpopulation. This is the component that is typically avoided by existing attacks, which instead focusing on optimizing the attack data in order to best achieve an arbitrarily chosen objective. However, the selection strategy is important: in order to keep collateral small, one must select a subpopulation which can be separated from the rest of the distribution. Otherwise, regions of the distribution which are misclassified will include samples not satisfying the filter function. We propose two subpopulation selection approaches, called \fmatch\ and \cmatch.

\subsubsection{\fmatch}
This filter function aims at matching some set of specific features of the data, that the adversary may be interested in targeting a priori. To use this filter, in addition to having 
 access to realistic data points $x_i$ and labels $y_i$, the adversary must have access to a set of annotations $a_i$ on these points.
These annotations represent structure that is finer grained or separate from the labels, such as race or age values for gender classification or color of automobiles in CIFAR-10. The annotations can be created manually by the adversary, or by selecting a subset of features in a tabular dataset (as we will see with the UCI Adult dataset).
\fmatch\ (Algorithm~\ref{alg:fmatch}) simply matches on the exact value of these annotations to identify the subpopulations. This is related to the standard definition for notions of group fairness in the algorithmic fairness literature~\cite{dwork2012fairness, hardt2016equality, buolamwini2018gender}.

\sloppy
Some attacks from the literature have run attacks like \fmatch. For example, Bagdasaryan et al.~\cite{bagdasaryan2018backdoor} attack CIFAR-10 models to target manually selected green cars. However, the attack considered in that work was not designed to generalize - they target a specific set of images, but don't evaluate on a holdout set of green cars to see how generalizable the attack was. By contrast, we measure \fmatch\ on how well it manages to attack test points from the subpopulation, ensuring that the attack is targeting the subpopulation as a whole and not those specific examples. Another example is the work of Schuster et al.~\cite{schuster2020humpty}, who propose attacks which can, as an application, attack machine translation, compromising all translations containing a specific word. This is a \fmatch\ attack, and they design a poisoning attack strategy specifically for this task. 
We will demonstrate that the generic attacks we develop here function on text classifiers trained using BERT as well.

\begin{algorithm}[t]
\begin{algorithmic}
\State Input: $X$ - features, $C$ - manual subpopulation annotations, $c_{att}$ - target subpopulation annotation
\State \Return $\filtfunc=\lambda~{x,y,c}:\mathbbm{1}(c=c_{att})$
\end{algorithmic}
\caption{\fmatch~Algorithm - leverage  data annotations}
\label{alg:fmatch}
\end{algorithm}

\subsubsection{\cmatch}
\label{sec:cmatch}
\begin{algorithm}[t]
\begin{algorithmic}
\State Input: $X \in D_{aux}$ - feature values;\hspace{0mm} $k_{\text{cluster}}$ - number of clusters; \hspace{0mm} $\textsc{PreProcess}$ - preprocessing function
\State $\text{centers} = \textsc{Cluster}(\textsc{PreProcess}(X), k_{\text{cluster}})$
\State $\text{target} = \textsc{PickCluster}(\text{centers})$
\State \Return $\filtfunc = \lambda~x: \mathbbm{1}(\textsc{ClosestCenter}(\textsc{PreProcess}(x),$
\State $\text{centers})==\text{target})$

\end{algorithmic}
\caption{\cmatch~Algorithm - automatically identify subpopulations}
\label{alg:cmatch}
\end{algorithm}
Our next filter function, \cmatch\ (Algorithm~\ref{alg:cmatch}), replaces the need for annotation with clustering to identify subpopulations of interest. By identifying natural clusters in the data, one can compromise the model for one cluster but not elsewhere. In \cmatch\ the attacker uses the auxiliary dataset $D_{aux}$ for clustering and identifying the most vulnerable subpopulations. 

There are various design decisions that need to be taken care of before we can use \cmatch. We must specify a preprocessing function $\textsc{PreProcess}$ applied to the auxiliary data, and a clustering algorithm. For preprocessing phase, we first use the representation layer of a neural network trained on $D_{aux}$ (we test which layer is most effective in Section~\ref{sec:exp}), and then apply a PCA projection. For clustering, we use KMeans, but any procedure for generating meaningful clusters on a given dataset should work. 



Interestingly, \cmatch\ can also be used in cases in which the adversary has a targeted subpopulation in mind.  For example, consider an adversary who wishes to disrupt street sign detection in a self-driving car through a subpopulation attack---\cmatch\ would help identify vulnerable street signs which will be easiest to target, increasing the impact and stealth of their attack. In general, an adversary can generate a clustering and identify a cluster that is both aligned with their goals and will be easy to attack. We show in the Appendix~\ref{app:examples} that subpopulations generated with \cmatch\ can be semantically meaningful.

\subsection{Poisoning Attack Generation}
\label{sec:att_make}

\subsubsection{Label Flipping}
\label{ssec:randflip}
For our poisoning attack generation, we begin by adapting a common baseline algorithm, label flipping, to our setting. Label flipping has been used in poisoning availability attacks~\cite{xiao2015feature} to create poisoning points that have similar feature values with legitimate data, but use a different label.

If the subpopulation size is $m$, and the adversary uses a poisoning rate $\alpha$ relative to the subpopulation, they add $\alpha m$ poisoned points, which should be small relative to the entire dataset size. In label flipping attacks, these points are generated by sampling $\alpha m$ points satisfying the filter function from $D_{aux}$ and adding these to the training set with a label $t$ different from the original one $c$. We choose a single label for the whole subpopulation, which maximizes the loss on the poison point.
Label flipping ensures high target damage, while the filter function itself is what guarantees low collateral---if it is a good enough separation, then the learning algorithm will be able to independently learn the poisoned subpopulation, without impacting the rest of the distribution. 

While simple, this attack is very general and applicable to various data modalities, including images, text, and tabular data, as our experiments will show.
To demonstrate the modularity of our attack framework, we show that leveraging optimization techniques such as influence functions, following the results of \cite{koh2017understanding}, and gradient optimization, can improve the effectiveness of our attacks. 

\subsubsection{Attack Optimization}
\label{ssec:att_opt}
In order to optimize points generated by label flipping, we follow Koh and Liang~\cite{koh2017understanding}.
They propose influence functions to understand the impact of training samples on trained models, and demonstrates as an application an optimization procedure for generating poisoning data.
To increase the loss on a test point $x_{test}, y_{test}$, by modifying a training point $x, y$ by a $\delta$ perturbation $x+\delta, y$, \cite{koh2017understanding} use gradient descent on the influence $\mathcal{I}$, with the following update, to optimize poisoning data:
\begin{equation}
    \nabla_x \mathcal{I}(x_{test}, y_{test}) = -\nabla_{\theta} L(x_{test}, y_{test}, \theta)^T H_{\theta}^{-1} \nabla_x \nabla_{\theta} L(x, y, \theta)
\end{equation}
This is derived by approximating the model's loss function with a quadratic function around the unpoisoned parameters $\theta$. The curvature of the quadratic function is described by the Hessian around $\theta$, $H_{\theta}$. In order to use the Hessian for even small models and datasets, \cite{koh2017understanding} use various optimizations and approximations; using Hessians is computationally expensive, and dominates the attack's running time. Indistinguishable attacks are a goal of \cite{koh2017understanding}; subpopulation attacks do not have this limitation: we only require that inputs remain valid (for images, pixels remain in [0, 255] bounds).

We will experiment with attacks based on influence functions, which we extend to subpopulations (they were originally used for targeted poisoning). We also propose a \emph{Gradient Optimization} attack or GO, in which we relax the quadratic approximation to a linear approximation, trading off  effectiveness for increased efficiency. This approximation assumes the poisoned model is computed as a single gradient step from an unpoisoned model, where the gradient is computed on both the clean training set $X, Y = \lbrace x_i, y_i\rbrace_{i=1}^n$ and poisoned training set $X_p, Y_p = \lbrace x^p_i, y^p_i\rbrace_{i=1}^m$:
\begin{equation}
    \begin{split}
        \theta_{new} = \theta_{old} & - \frac{\eta}{n+m}\sum_{i=1}^{n}\nabla_{\theta}L(x_i, y_i, \theta_{old})\\ & - \frac{\eta}{n+m}\sum_{i=1}^m\nabla_{\theta}L(x^p_i, y^p_i, \theta_{old}).
    \end{split}
    \label{eq:onestep}
\end{equation}

We hope to induce a modification $\theta_{new} - \theta_{old}$ that will most increase the loss on a target dataset $X_a, Y_a$ selected from $D_{aux}$. As $\eta\rightarrow 0$, the best such modification is $-\nabla_{\theta}L(X_a, Y_a, \theta)$. Then our goal is to maximize $Obj(X_p, Y_p) = -\nabla_{\theta}L(X_a, Y_a, \theta)^T(\theta_{new} - \theta_{old})$. Combining this with Equation~\ref{eq:onestep}, our optimization process is:
\[
X_p, Y_p = \argmax_{X_p, Y_p}\nabla_{\theta}L(X_a, Y_a, \theta_{old})\nabla_{\theta}L(X_p, Y_p, \theta_{old}).
\]

To solve this, we follow standard procedure by using gradient descent. We start with a poisoning set generated by label flipping, and run 50 steps of gradient descent to update $X_p, Y_p$. The gradient update on $\nabla_{X_p}$ for this optimization procedure is
\[
-\nabla_{\theta}L(X_a, Y_a, \theta_{old})^T\nabla_{X_p}\nabla_{\theta}L(X_p, Y_p, \theta_{old}),
\]
which is equivalent to approximating the influence function update with $H_{\theta}=I$. Intuitively, this allows us to maintain good performance without the expensive Hessian operations required for influence functions. There are many other optimization-based procedures for poisoning attacks (e.g., \cite{Biggio2012PoisoningAA, shafahi2018poison, geiping2020witches}), differing in objective function, threat model, and initialization strategy. It may be possible to adapt these strategies to subpopulation attacks, as well.

A gradient-based optimization approach assumes the features can be modified as continuous values. For tabular datasets with categorical features and feature dependencies, or text datasets, these continuous modifications are difficult, and even generating adversarial examples requires a significant amount of work. The original influence attack~\cite{koh2017understanding} was tested only on image data. For this reason, we implement optimization-based attacks based on influence and gradient optimization for image datasets only. We believe that extending existing optimization attacks to text or tabular data~\cite{koh2018stronger} is an interesting avenue of future work. 

\subsection{Subpopulation Attack Framework}
\begin{algorithm}[t]
\begin{algorithmic}
\State Input: Adversarial knowledge $K_{adv}$, attack size $m$
\State $\text{FilterFunctions}=\textsc{MakeFilterFunctions}(K_{adv})$ \Comment{e.g., Algorithm~\ref{alg:fmatch} or~\ref{alg:cmatch}}
\State $\mathcal{F} = \textsc{SelectFilterFunction}(\text{FilterFunctions}, K_{adv})$
\State \Return $\textsc{GenerateAttack}(\mathcal{F}, m, K_{adv})$ \Comment{see Section~\ref{sec:att_make}}
\end{algorithmic}
\caption{Generic Subpopulation Attack. In this work, $K_{adv}$ consists only of a the dataset $D_{aux}$.}
\label{alg:generalsubatt}
\end{algorithm}

The full attack operates in three distinct steps. It starts by identifying subpopulations in the data with \fmatch\ or \cmatch, continues by selecting a target subpopulation, and ultimately generating the actual poisoned points to inject into the training set, as shown in Figure~\ref{fig:overview} and Algorithm~\ref{alg:generalsubatt}.
In order to generate the contaminants, the adversary could use a basic, yet generic, approach such as label flipping, or employ a stronger model-specific approach like influence functions or gradient optimization.
Similarly, different clustering methods could be leveraged to identify potential victim subpopulations in the first step.
Due to the modularity of the framework, each component (e.g., subpopulation selection, poisoning points generation) may be improved separately, and in ways that may more closely match specific domains of interest.

\section{Attack Experiments}
\label{sec:exp}
In this section, we explore the threat of subpopulation attacks on real datasets. 
We first explore the effectiveness of the label flipping attack in the end-to-end training scenario, comparing \fmatch\ and \cmatch\ for subpopulation selection. Then, we show that label flipping subpopulation attacks are also effective in transfer learning scenarios. We run experiments on four datasets from three modalities, to demonstrate the generality of our attack.
For three of these datasets, we measure the attack's success on both small and large models. Large models are fine tuning all the layers of the neural networks and have much larger capacity (e.g., 134 million parameters for VGG-16). We also evaluate the two optimization attacks on the face recognition dataset.
We believe the breadth of our experiments provides compelling evidence that subpopulation attacks are a useful and practical threat model for poisoning attacks against ML. 


We evaluate our attacks using a general approach. We partition standard datasets into a training set $D$, an auxiliary set $D_{aux}$, and a test set $D_t$, all being disjoint. The adversary only ever has access to $D_{aux}$. The adversary uses $D_{aux}$ to generate subpopulations, training a surrogate model when necessary for \cmatch. The adversary generates the poisoned data $D_p$, and the model is trained on $D \cup D_p$, and the target damage is evaluated only on test points from $D_t$ belonging to the target subpopulation. 



\subsection{Datasets and Models}
Here we will provide a brief overview of the various datasets, models, and hyperparameters used for our experiments.

\noindent
\textbf{CIFAR-10.} CIFAR-10~\cite{Krizhevsky09learningmultiple} is a medium-size image classification dataset of 32x32x3 images split into 50000 training images and 10000 test images belonging to one of 10 classes. We use CIFAR-10's standard split, splitting the train set into 25000 points for $D$ and 25000 points for $D_{aux}$.
We use two ML classifiers: \scifar\ and \lcifar. \scifar\ is a small convolutional neural network, consisting of a convolutional layer, three blocks consisting of 2 convolutional layers and an average pooling layer each, and a final convolution and mean pooling layer, trained with Adam at a learning rate of 0.001. \lcifar\ fine tunes all layers of a VGG-16 model pretrained on ImageNet for 12 epochs with Adam with a learning rate of 0.001.

\noindent
\textbf{UTKFace.} UTKFace~\cite{zhifei2017cvpr} is a facial recognition dataset, annotated with gender, age, and race (the races included are White, Black, Asian, Indian, and Other, which contains Latino and Middle Eastern images). We use it for gender classification, removing children under age 15 to improve performance, leaving 20054 images. We then split it into $D$,  $D_{aux}$, and $D_{t}$ with 7000, 7000, and 6054 images, respectively.
For UTKFace, we use two ML models: \sutk\ and \lutk. \sutk\ only trains the last layer of a VGG-16 model~\cite{simonyan2014very}, pretrained on the ImageNet dataset~\cite{imagenet_cvpr09} (8K parameters), while \lutk\ trains all layers (134M parameters). For both, we train for 12 epochs with Adam, using a learning rate of 0.001 for \sutk\ and 0.0001 for \lutk. For both, we use an $\ell_2$ regularization of 0.01 on the classification layer to mitigate overfitting.

\noindent
\textbf{IMDB Reviews.} The IMDB movie review dataset~\cite{maas-EtAl:2011:ACL-HLT2011} consists of 50000 reviews of popular movies left by users on the IMDB\footnote{\url{https://www.imdb.com/}} website, together with the reviews' scores. The dataset is often used for binary sentiment classification, predicting whether the review expresses a positive or negative sentiment.
Given the rise in popularity of pre-trained models (BERT~\cite{devlin2018bert}, GPT-2~\cite{radford2019language}, XLNet~\cite{yang2019xlnet}) for natural language modeling, and the high variance of the data used to train them, they provide a perfect target for subpopulation attacks.
For this dataset we split the training set into 12500 points for $D$ and 12500 for $D_{aux}$.
We use BERT for our experiments, followed by a single classifier layer. Our small model, \sbert, only fine-tunes BERT's final transformer block and the last layer (classifier), while our large model, \lbert, fine-tunes all of BERT's transformer blocks together with the classifier.
Both models are trained on vectors of 256 tokens, for 4 epochs, with a learning rate of $10^{-5}$ and mini batch size of 8.
These models use the same architecture and implementation from the Huggingface Transformers library \cite{Wolf2019HuggingFacesTS}.
They are both based on a pre-trained \verb|bert-base-uncased| instance~\cite{devlin2018bert}, with 12 transformer blocks (110M parameters), and one linear layer for classification (1538 parameters).

\noindent
\textbf{UCI Adult.} UCI Adult~\cite{Dua:2019} consists of 48843 rows, where the goal is to use demographic information to predict whether a person's income is above $\$50K$ a year. We drop the 'education', 'native-country', and 'fnlwgt' columns due to significant correlation with other columns, and apply one-hot encoding to categorical columns.
For UCI Adult, we use a feed-forward neural network with one hidden layer of 10 ReLU units, trained for a maximum of 3000 iterations using scikit-learn~\cite{scikit-learn} default settings for all other parameters. 

\begin{table}[]
    \renewcommand\arraystretch{1.15}
    \centering
    \begin{tabular}{c c | c c c c}
        \hline
        \multirow{2}{*}{Dataset} & \multirow{2}{*}{Worst} & \multirow{2}{*}{Clean Acc} & \multicolumn{3}{c}{Target Damage} \\
         &  &  & $\alpha=0.5$ & $\alpha=1$ & $\alpha=2$ \\
        \hline
        \multirow{3}{*}{\shortstack{UTKFace \\ \sutk}} & 10 & \multirow{3}{*}{0.846} & 0.054 & 0.086 & 0.144 \\
         & 5 & ~ & 0.094 & 0.140 & 0.192 \\
         & 1 & ~ & 0.400 & 0.400 & 0.400 \\
        \hline
        \multirow{3}{*}{UCI Adult} & 10 & \multirow{3}{*}{0.837} & 0.103 & 0.148 & 0.16 \\
         & 5 & ~ & 0.143 & 0.21 & 0.195 \\
         & 1 & ~ & 0.311 & 0.467 & 0.250 \\
    \end{tabular}
    \caption{Clean accuracy and target damage for \fmatch\ attacks with label flipping, reported over the worst 1, 5, and 10 subpopulations.}
    \label{tab:fmatch}
    \vspace{-5pt}
\end{table}

\begin{table}[t]
    \renewcommand\arraystretch{1.15}
    \centering
    \begin{tabular}{c | ccccc}
        \hline
        Task & Input & Layer 1 & Layer 2 & Layer 3 & Layer 4 \\
        \hline
        Worst-5 & 0.187 & 0.290 & 0.313 & 0.305 & 0.327 \\
        Worst-10 & 0.157 & 0.257 & 0.270 & 0.283 & 0.297 \\
    \end{tabular}
    \caption{Results for \cmatch\ with Label Flipping on CIFAR-10 + \scifar\ using six different layers for clustering. Both worst-5 and worst-10 results improve monotonically in the layer number, indicating that the more useful the representation is to the model, the easier it is to attack, as well.}
    \label{tab:layers}
    \vspace{-5pt}
\end{table}

\subsection{Label Flipping on End-to-end Training}

We demonstrate that subpopulation attacks based on label flipping are effective for end-to-end training of neural networks, which has been a difficult setting for poisoning attacks on neural networks. Backdoor poisoning attacks need a fairly large amount of poisoning data to be effective~\cite{gu2017badnets}, while targeted attacks against  neural networks trained end-to-end  need on the order of 50 points to attack a single point at testing~\cite{shafahi2018poison, geiping2020witches}.  We investigate our attacks on CIFAR-10 and UCI Adult. We train \scifar\ on CIFAR-10 and the UCI Adult model end-to-end. We test \fmatch\ on UCI Adult, with the combination of education level, race, and gender as annotations. We also use KMeans for \cmatch\ on both UCI Adult and CIFAR-10 with $k=100$ clusters (we discuss briefly results for other values of $k$). For UCI Adult, due to large variance in \fmatch\ subpopulation sizes, we remove subpopulations with greater than 100 or less than 10 data points. We report \fmatch\ results in Table \ref{tab:fmatch}, and \cmatch\ results in Table \ref{tab:cmatch}.

\paragraph{Effectiveness of Attacks}
\fmatch\ causes a target damage of 25\% for one subpopulation on UCI Adult and an average of 19.5\% over five subpopulations, when attacked with a poisoning rate of $\alpha=2$. The attacked groups all have comparable ages, genders, and races, impacting the group fairness of the classifier. There are some results where increasing $\alpha$ does not strictly increase attack performance, due likely to unintuitive effects from optimizing nonconvex loss functions (all of our results with convex loss functions have monotone performance increases).

We can only run \fmatch\ on datasets with annotations: UTKface and UCI Adult. We find \cmatch\ is nearly always more effective than \fmatch. On UCI Adult, with a poisoning rate of $\alpha=2$, one subpopulation reaches 66.7\% target damage (with 48 poisoned points), and five subpopulations reach an average of 36.7\% target damage (with an average of 45 poisoned points). On CIFAR-10, the attack is very effective, reaching 23.6\% target damage on one subpopulation at only a poisoning rate of $\alpha=.5$ and $53.5\%$ at $\alpha=2$. The collateral is also low, on average 1.41\% for the top 5 subpopulations by target damage. This is convincing evidence that \cmatch's ability to leverage the data to construct subpopulations allows it to produce more effective subpopulations. 

\paragraph{Optimizing \cmatch}
We experiment with designing good \cmatch\ clusters on CIFAR-10, by running KMeans to produce 100 subpopulations at six different layers of \scifar. For all layers, we project to 10 dimensions using PCA. We find that \cmatch\ is most effective when using clusters constructed using the last layer of \scifar. Notably, the top 5 clusters using the last layer representation achieve 32.7\% target damage, while \cmatch\ using the input features only reaches 18.7\% target damage, as shown in Table~\ref{tab:layers}. There is a monotonic increase in the target damage as the layer number increases, indicating, intuitively, that the more useful representations to the model are more effective for subpopulation attacks, as well. We use this insight for all of our remaining experiments. 

\paragraph{Subpopulation size} We also varied the number of clusters within $\lbrace 50, 100, 200, 400\rbrace$, and found little significant difference between these values. If the number of clusters is too small, however, the impact on overall accuracy may be noticeable, and if the number of clusters is too large, the subpopulation impacted will consist of very few points, making the impact of the attack itself limited (at 400 subpopulations, for example, some subpopulations do not appear in the test set). We decided to report all results for $k=100$ clusters. 

\subsection{Label Flipping on Transfer Learning}

\begin{table}[]
    \renewcommand\arraystretch{1.2}
    \centering
    \begin{tabular}{c | cccc}
        \hline
        Models & CIFAR-10 & UCI Adult & UTKFace & IMDB \\
        \hline
        Small & 1.4\% & 1.4\% & 0.3\% & -0.29\% \\
        Large & 1.3\% & N/A & 2.9\% & 0.53\% \\
    \end{tabular}
    \caption{Mean collateral for Label Flipping poisoning attacks on worst 5 subpopulations by target damage with $\alpha=1$. Small models: \scifar\ for CIFAR-10, \sutk\ for UTKFace, \sbert\ for IMDB. Large models: \lcifar\ for CIFAR-10, \lutk\ for UTKFace, \lbert\ for IMDB.}
    \label{tab:collat}
    \vspace{-5pt}
\end{table}

\begin{table}[t]
    \renewcommand\arraystretch{1.15}
    \centering
    \tabcolsep=0.15cm
    \begin{tabular}{c c | c c c c c}
        \hline
        \multirow{2}{*}{Dataset} & \multirow{2}{*}{Worst} & \multirow{2}{*}{Clean Acc} & \multicolumn{3}{c}{Target Damage} & \multirow{2}{*}{\shortstack{Subpop \\Size}} \\
         &  &  & $\alpha=0.5$ & $\alpha=1$ & $\alpha=2$ & ~ \\
        \hline
        \multirow{3}{*}{\shortstack{UTKFace \\ \sutk}} & 10 & \multirow{3}{*}{0.846} & 0.073 & 0.128 & 0.294 & 60.7 \\
         & 5 & ~ & 0.094 & 0.132 & 0.335 & 47.8 \\
         & 1 & ~ & 0.222 & 0.222 & 0.556 & 40.3 \\
        \hline
        \multirow{3}{*}{\shortstack{IMDB\\ \sbert}} & 10 & \multirow{3}{*}{0.889} & 0.010 & 0.018 & 0.036 & 158.4 \\
         & 5 & ~ & 0.014 & 0.029 & 0.061 & 171.3 \\
         & 1 & ~ & 0.038 & 0.049 & 0.203 & 207.3 \\
        \hline
        \multirow{3}{*}{\shortstack{CIFAR-10 \\ \scifar}} & 10 & \multirow{3}{*}{0.803} & 0.136 & 0.297 & 0.418 & 111.5 \\
         & 5 & ~         & 0.163 & 0.327 & 0.480 & 120.9 \\
         & 1 & ~         & 0.236 & 0.348 & 0.535 & 126.8 \\
        \hline
        \multirow{3}{*}{UCI Adult} & 10 & \multirow{3}{*}{0.837} & 0.121 & 0.24 & 0.224 & 41.1 \\
         & 5 & ~ & 0.191 & 0.343 & 0.367 & 45.1 \\
         & 1 & ~ & 0.333 & 0.375 & 0.667 & 48.3 \\
    \end{tabular}
    \caption{Clean accuracy and target damage for small models attacked with \cmatch\ with Label Flipping, reported over the worst 1, 5, and 10 subpopulations.
    On BERT we measure the performance sampling 10 clusters at the lowest, medium, and highest confidence levels,
    due to running time constraints. Subpopulation sizes (the column Size) are averages over poison rates.} 
    \label{tab:cmatch}
    \vspace{-5pt}
\end{table}

\begin{table}[]
    \renewcommand\arraystretch{1.15}
    \tabcolsep=0.15cm
    \centering
    \begin{tabular}{c c | c cc cc cc c}
        \hline
        \multirow{2}{*}{Dataset} & \multirow{2}{*}{Worst} & \multirow{2}{*}{Clean Acc} & \multicolumn{3}{c}{Target Damage} & \multirow{2}{*}{Size} \\
         &  &  & $\alpha=0.5$ & $\alpha=1$ & $\alpha=2$ & ~ \\
        \hline
        \multirow{3}{*}{\shortstack{UTKFace \\ \lutk}} & 10 & \multirow{3}{*}{0.963} & 0.218 & 0.329 & 0.405 & 57.3 \\
         & 5 & ~ & 0.244 & 0.385 & 0.432 & 38.1 \\
         & 1 & ~ & 0.286 & 0.500 & 0.455 & 29.0 \\
        \hline
        \multirow{3}{*}{\shortstack{IMDB \\ \lbert}} & 10 & \multirow{3}{*}{0.913} & 0.024 & 0.080 & 0.206 & 148.5 \\
         & 5 & ~ & 0.035 & 0.129 & 0.303 & 136.2 \\
         & 1 & ~ & 0.051 & 0.204 & 0.506 & 129.0 \\
        \hline
        \multirow{3}{*}{\shortstack{CIFAR-10 \\ \lcifar}} & 10 & \multirow{3}{*}{0.863} & 0.206 & 0.518 & 0.511 & 175.6 \\
         & 5 & ~ & 0.294 & 0.616 & 0.627 & 180.9 \\
         & 1 & ~ & 0.426 & 0.738 & 0.742 & 144.0 \\
    \end{tabular}
    \caption{Clean accuracy and target damage for large models trained on datasets that have been attacked with \cmatch\ and Label Flipping, reported over the worst 1, 5, and 10 subpopulations.
    On BERT we measure the performance sampling 10 clusters at the lowest, medium, and highest confidence levels,
    due to running time constraints. Subpopulation sizes are averages over poison rates.
    These attacks are often very damaging: 10 subpopulations from UTKFace reach an average target damage of 40.5\%, and 10 subpopulations on CIFAR-10 reach an average target damage of 51.1\%. Results on IMDB are also markedly better with \lbert\ than on \sbert, with little collateral damage.
    }
    \label{tab:cmatch-big}
    \vspace{-5pt}
\end{table}

We evaluate the effectiveness of Label Flipping subpopulation attacks on transfer learning, on CIFAR-10, UTKFace, and IMDB. We run \fmatch\ on \sutk\ on UTKFace, using the combination of race and bucketed ages (bucketed into $[15, 30]$, $[30, 45]$, $[45, 60]$, $[60, \infty)$) as the annotations, presenting results in Table \ref{tab:fmatch}. We present results for \cmatch\, with 100 clusters generated with KMeans, on \sutk\ on UTKFace and \sbert\ in Table \ref{tab:cmatch}, and present results for large models in Table \ref{tab:cmatch-big}.



We find \fmatch\ is effective on UTKFace, reaching a target damage of 40\% for one subpopulation and an average of 19.2\% over five subpopulations, both at a poisoning rate of $\alpha=2$. However, \cmatch\ reaches a target damage of 55.6\% on one subpopulation and an average of 33.5\% over five subpopulations. This reiterates the results from end-to-end training on UCI Adult: \cmatch\ tends to outperform \fmatch\ when both are applicable, but both are effective.
On the IMDB data with \sbert, \cmatch\ appears to rather ineffective until the poisoning attack gets large, causing a target damage of 20.3\% in one subpopulation, when $\alpha=2$, but under 5\% when $\alpha=0.5$ or $\alpha=1$. However, on \lbert\ the attack has markedly better performance, with peak target damage of $\approx 20\%$ already at $\alpha=1$, which grows to an average of $30.3\%$ for $\alpha=2$, over five subpopulations, peaking at $50.6\%$ on the most effected one. The maximum registered collateral 
damage for these attacks was 0.16\% and 0.99\%, for \sbert\ and \lbert\ respectively, with
$\alpha=2$.

For other large models, we find that the attacks are also very effective, despite the higher base accuracy, with many subpopulations also from the UTKFace and CIFAR-10 datasets significantly impacted. With an average of only 38.1 poison points, five subpopulations from UTKFace have an average target damage of 38.5\% (with $\alpha=1$), and still has a significant target damage of 24.4\% with only 19.1 poison points on average ($\alpha=.5$). On CIFAR-10, we observe particularly damaging attacks: five subpopulations, when attacked with $\alpha=1$, reach an average target damage of 61.6\%. Generally, attacks on the larger models are more effective than attacks on smaller models. This may be an indication that larger models learn more subpopulation-specific decision boundaries.

We present in Table~\ref{tab:collat} the collateral of our poisoning attacks, measuring the mean collateral damage over the worst-5 subpopulations by target damage. The worst error decay occurs in UTKFace with the fine tuned model, decreasing the accuracy by 2.9\% accuracy while these subpopulations had accuracy decrease by 38.5\%. Meanwhile, the collateral is typically around 1.5\% for other datasets and models, and is insignificant for the last layer transfer learned model on UTKFace, and for both models on the IMDB dataset.

\subsection{Fairness Case Study}
\label{sec:fairness}
Our work is a proof-of-concept of poisoning attacks impacting algorithmic fairness, a goal which has been attempted directly by~\cite{chang2020adversarial, solans2020poisoning}. This is most immediately seen when considering specific examples of \fmatch\ attacks. For example, on UCI Adult, the subpopulation consisting of Black women high school graduates drops from an accuracy of 91.4\% to 76.7\% when the attack has size $\alpha=2$. On UTKFace, the subpopulation which is easiest to attack consists of people >60 years old and who are Latino/Middle Eastern; this subpopulation's accuracy drops from 100\% to 60\%. Images containing white people from 30-45 years old have their accuracy drop by 15.2\%, in contrast. Some subpopulations generated by \cmatch\ may also have fairness implications due to correlation with sensitive attributes, shown in Appendix~\ref{app:examples}.

Aside from being a fairness concern, it is interesting that the model is able to be compromised by \fmatch\ on UTKFace; here, the model only has access to pixels, and does not explicitly use demographic information. This indicates these demographics are learned by the model independently of their correlation with the target class, highlighted as a privacy risk in prior work~\cite{song2019overlearning, mireshghallah2020principled}.

\subsection{Subpopulation Transferability}
\label{sec:transfer}
In Table~\ref{tab:transfer}, we present results attacking CIFAR-10 + \lutk\ with \cmatch\, with subpopulations generated with ResNet-50 embeddings. This allows us to test whether knowledge of the learner's model architecture is required to run \cmatch\ attacks. We find that the attack is less successful when using transferred embeddings, but is still effective --- most subpopulations exhibit 15-16\% target damage, while causing no more than 2\% collateral damage. While knowledge of architecture is helpful for generating \cmatch\ subpopulations, it is not necessary.

\begin{table}[]
    \renewcommand\arraystretch{1.15}
    \centering
    \begin{tabular}{c c | c c c c}
        \hline
        \multirow{2}{*}{Data + Model} & \multirow{2}{*}{Worst} & \multicolumn{3}{c}{Target Damage} & \multirow{2}{*}{Size} \\
         &  & $\alpha=0.5$ & $\alpha=1$ & $\alpha=2$ & ~ \\
        \hline
        \multirow{3}{*}{CIFAR-10 + \lcifar} & 10 & 0.156 & 0.155 & 0.159 & 111.5 \\
         & 5 & 0.159 & 0.157 & 0.161 & 119.3\\
         & 1 & 0.165 & 0.161 & 0.166 & 157.3 \\
    \end{tabular}
    \caption{Target damage for CIFAR-10 large models attacked with $\cmatch$ clusters and label flipping generated with embeddings transferred from a ResNet-50 model.}
    \label{tab:transfer}
    \vspace{-5pt}
\end{table}

\subsection{Attack Point Optimization}
So far, our evaluation has focused on the label flipping attack, which generalizes across many data modalities. However, for image datasets, where we have the ability to continuously modify features, we will show that there is a benefit to optimizing the feature values to improve the attack. Because gradient-based optimization is typically designed for continuous-value features~\cite{koh2017understanding, geiping2020witches}, we run only on UTKFace.

We summarize our attack optimization results in Table~\ref{tab:utk_opt}. We find that Gradient Optimization causes significant performance improvements on small models, nearly doubling the performance of attacks on UTKFace + \sutk. For example, with $\alpha=1$, the label flipping attack achieves an average of 8.4\% target damage, while Gradient Optimization increases this to 15.5\% target damage. With influence, this undergoes another doubling to 38.6\%. However, for large models, the Gradient Optimization does not perform as well, not improving the average performance on UTKFace + \lutk, and often even decreasing it. This corroborates recent results indicating that influence functions break down on deep models~\cite{basu2020influence}. Because our attack is an approximation to influence functions, it is natural that our attack is also brittle in this regime.

The running time of gradient optimization significantly outperforms influence. For \sutk, gradient optimization approach returns poisoned data (from 50 iterations of optimization) in roughly 40 seconds, while influence takes on average 5.4 hours (to complete 250 iterations, one eighth the number used in \cite{koh2017understanding}). Although unsuccessful, gradient optimization with \lutk\ takes roughly 55 seconds, while it timed out for influence (when provided with a 5 minute timer for the first iteration, which would result in >1 day per attack). There is a clear tradeoff between running time and performance when deciding between influence and gradient optimization, although both break down at larger model sizes.

\begin{table}[h]
    \renewcommand\arraystretch{1.15}
    \centering
    \begin{tabular}{cc|ccc}
        \hline
        \multirow{2}{*}{Data + Model} & \multirow{2}{*}{Att} & \multicolumn{3}{c}{Target Damage} \\
         &  & $\alpha=0.5$ & $\alpha=1$ & $\alpha=2$ \\
        \hline
        \multirow{3}{*}{UTKFace + \sutk} & LF & 0.032 & 0.084 & 0.122 \\
        ~ & GO & 0.049 & 0.155 & 0.206 \\
        ~ & INF & 0.238 & 0.386 & 0.352 \\
        \hline
        \multirow{2}{*}{UTKFace + \lutk} & LF & 0.134 & 0.235 & 0.318 \\
        ~ & GO & 0.131 & 0.182 & 0.185 \\

    \end{tabular}
    \caption{Optimization Approaches for Poisoning Attack Generation: LF = Label Flipping, GO = Gradient Optimization, INF = Influence (250 iterations). We select a subset of 10 subpopulations to demonstrate the performance of the attack, measuring mean performance across these subpopulations. The attack performance nearly doubles for the small model with GO, but decreases for large models. INF improves the performance significantly, but is very slow.}
    \label{tab:utk_opt}
    \vspace{-5pt}
\end{table}

\section{Improving Targeted Attacks Using Subpopulations}
\label{sec:targeted}
Here, we demonstrate that identifying subpopulations can be complementary to targeted attacks---if an adversary's goal is to target a fixed set of $k$ inputs, inputs from a subpopulation will be easier to target than arbitrarily selecting them. To illustrate this, we use two state-of-the-art targeted attacks, influence functions~\cite{koh2017understanding} and Witches' Brew~\cite{geiping2020witches}. Influence functions are powerful and assume a large amount of adversarial knowledge, while Witches' Brew is less powerful but operates in a realistic threat model.
In this section, we keep the targeted attack algorithms unchanged, and only alter the points they target. We show that, when the $k$ points are selected from a subpopulation generated with \cmatch, they are more easily targeted than when $k$ points are selected at random. 

\subsection{Influence Functions~\cite{koh2017understanding}}
Influence functions were proposed for a targeted poisoning attack with a single target point~\cite{koh2017understanding}, but Figure 5 from \cite{koh2017understanding} also shows that influence functions can attack multiple test points at once.
To run this attack, the authors manually collect a set of 30 images of one author's dog, and show that a misclassified image of a dog in training can be used to attack this set of dogs.
To compare directly with this experiment, we will show that using \cmatch\ to generate this set of target examples makes the attack easier than arbitrarily selecting a set of examples by hand.
To show this, we start with the images of dogs from the test set of the Kaggle Cats and Dogs dataset\footnote{https://www.kaggle.com/chetankv/dogs-cats-images}.
We then select a set of targets of the same size (30 points), either by randomly sampling (to simulate the arbitrary selection procedure used by \cite{koh2017understanding}) or selecting 30 points from a subpopulation generated with \cmatch.
We generate 10 sets of targets for both strategies, and attack all target sets with a single data point (we use only 10 sets of targets due to the high computational cost of the influence optimization procedure).
We run influence optimization for 2000 iterations, to maintain consistency with~\cite{koh2017understanding}.
For each of these target sets, we measure both the effectiveness of the attack and the number of optimization iterations it takes to completely misclassify the target points (if the attack is successful).

The results of this experiment are presented in Table~\ref{tab:inf}.
While multiple target sets are resistant to attack with both random selection and \cmatch, \cmatch\ targets are easier to attack in the worst case (top-5 attack success is an average of 0.6, compared to 0.24 for random selection), and take fewer iterations to completely compromise when they are vulnerable (250 iterations, compared to 380 for random selection, and 2000 in the original paper). This iteration reduction corresponds to reducing the running time of a successful attack from roughly 4 hours in the original paper to 30 minutes for \cmatch\ (and roughly 45 minutes for random selection). Interestingly, the test points we drew from for this experiment were easier to attack than those used in \cite{koh2017understanding}, likely due to a disparity in the distribution of dogs from the ImageNet dataset and Kaggle Cats and Dogs.

\begin{table}[h]
    \renewcommand\arraystretch{1.15}
    \centering
    \begin{tabular}{c| c c c}
        \hline
        Selection & Worst-1 & Worst-5 & Iterations (successful) \\
        \hline
        Koda \cite{koh2017understanding} & 0.53 & N/A & 2000 \\ 
        Random & 1.0 & 0.24 & 380 \\
        \cmatch & 1.0 & 0.6 & 250 \\
    \end{tabular}
    \caption{Effectiveness of \cmatch\ to aid influence-based poisoning attack generation. Compared to randomly generated baseline subpopulations, subpopulations generated with \cmatch\ are easier to poison (by worst-5 attack effectiveness), and take fewer iterations to attack (the number of iterations required to launch a successful attacks, when the attack is eventually successful).}
    \label{tab:inf}
    \vspace{-5pt}
\end{table}

\subsection{Witches' Brew~\cite{geiping2020witches}}

We additionally compare with the Witches' Brew attack of Geiping et al.~\cite{geiping2020witches}. This attack optimizes poisoning points using a gradient matching objective, and uses tools such as random restarts, differentiable data augmentation, and ensembling, which allow for very efficient targeted data poisoning attacks, which notably do not require fixing the model initialization, a requirement of other prior work~\cite{shafahi2018poison}. To demonstrate compatibility with subpopulations, we run a similar experiment as for influence functions---on CIFAR-10, we generate \cmatch\ subpopulations per-class using a clean model, and run targeted attacks where 30 points are chosen within a single subpopulation. We compare to a baseline of 30 randomly selected points from a single class. Subpopulations are generated per-class as Witches' Brew targets must have a uniform class. We run 24 attacks for each target selection procedure, using an ensemble size of 1, a budget of 1\% (500 poison points), differentiable data augmentation, 8 restarts, and $\varepsilon=16$, which is comparable to experiments in their Table 1. We report the fraction of targets misclassified by Witches' Brew in Table~\ref{tab:wb}, averaging performance over 8 models per trial. \cmatch -generated subpopulations outperform randomly selected points by a factor of 2.35x, when comparing the 10 most effective attacks. Over all 24 trials, the average \cmatch -based attack performs 86\% better than the average random selection attack.

\begin{table}[h]
    \renewcommand\arraystretch{1.15}
    \centering
    \begin{tabular}{c| c c c c}
        \hline
        Selection & Worst-1 & Worst-5 & Worst-10 & Overall \\
        \hline
        Random & 0.300 & 0.153 & 0.114 & 0.072 \\
        \cmatch & 0.951 & 0.382 & 0.264 & 0.134 \\
    \end{tabular}
    \caption{Effectiveness of Witches' Brew targeted poisoning, using arbitrary and \cmatch -based target selection.}
    \label{tab:wb}
    \vspace{-5pt}
\end{table}
\section{Impossibility of Defenses}
\label{sec:defenses}
We present here an impossibility result for defending against subpopulation attacks, based on a model of learning theory relying on mixture models from \cite{feldman2019does}. Informally, a $k$-subpopulation mixture distribution is a mixture of $k$ subpopulations of disjoint support, with mixing coefficients $\alpha_i, i \in \{1,\dots,k\}$. A subpopulation mixture learner is locally dependent if the learner makes local decisions based only on subpopulations. We consider binary classifiers that use the 0-1 loss.  We use a simplified version of the \cite{feldman2019does} model to prove the following result:



\begin{theorem}
For any dataset $D$ of size $n$ drawn from a  $k$-subpopulation mixture distribution, there exists a subpopulation poisoning attack $D_p$ of size $\le n/k$ that causes all locally dependent $k$-subpopulation mixture learners $\mathcal{A}$ return $\mathcal{A}(D\cup D_p)=\mathcal{A}(D)$ with probability $<1/2$. Additionally, if $\alpha$ is the weight of the smallest subpopulation in the mixture distribution, then a subpopulation attack of size less than $ 2\alpha n$ suffices with probability at least $ 1 - \exp(- \alpha n/2)$.
\end{theorem}

Essentially, if the learning algorithm makes subpopulation-wide decisions, it will inherently be susceptible to subpopulation attacks. \cite{feldman2019does} shows that this structure holds for $k$-nearest neighbors, mixture models, overparameterized linear models, and suggests (based on some empirical evidence) that it holds for neural networks as well. This result becomes more interesting when $k$ is large, as this case represents more diverse data and smaller attacks. In Appendix~\ref{app:proof}, we state our formal definitions, theorem statement, and proof. 

However, there are ways that this negative result may not reflect practice. First, the specific learning algorithm and dataset may  avoid making subpopulation-wide decisions. Our attack results corroborate this---not all subpopulations are easy to attack, and even at a poisoning rate of 2, no subpopulations drop to 0\% accuracy. Second, our negative result only applies to purely algorithmic defenses which do not involve any human input. A potential circumvention for this would be to use a diverse, carefully annotated validation dataset to identify ongoing subpopulation attacks. However, this defense requires careful data collection and annotation, as well as knowledge of the subpopulations that may be attacked. Ultimately, there may be no substitute for careful data analysis.

\section{Empirical Analysis of Existing Defenses}
\label{sec:defemp}

\begin{table*}
    \centering
    \small
    \tabcolsep=0.15cm
    \begin{tabular}{c | c c c | c c c | c c c | c c c}
        \hline
        \multirow{2}{*}{$\alpha$} & \multicolumn{3}{c}{Subpopulation 1} & \multicolumn{3}{c}{Subpopulation 2} & \multicolumn{3}{c}{Subpopulation 3} & \multicolumn{3}{c}{Subpopulation 4} \\
        ~ & Orig & TRIM & SEVER & Orig & TRIM & SEVER & Orig & TRIM & SEVER & Orig & TRIM & SEVER \\
        \hline
        0.5 & 0.16 & 0.01 & -0.02 & 0.12 & 0.12 & 0.13 & 0.12 & 0.03 & 0.08 & 0.12 & -0.07 & -0.01 \\
        1.0 & 0.26 & 0.16 & 0.21 & 0.18 & 0.18 & 0.10 & 0.18 & 0.15 & 0.02 & 0.14 & 0.06 & 0.04 \\
        2.0 & 0.48 & -0.02 & 0.17 & 0.31 & 0.06 & 0.17 & 0.29 & 0.19 & 0.18 & 0.24 & 0.26 & 0.18 
    \end{tabular}
    \caption{Comparing standard training with TRIM and SEVER for the 5 most heavily damaged \cmatch\ subpopulations on UTKFace + \sutk\ for each poisoning rate $\alpha$. Orig indicates target damage of undefended training, TRIM is target damage under TRIM, and SEVER is target damage under SEVER.}
    \label{tab:trim}
    \vspace{-5pt}
\end{table*}

\begin{table*}[]
    \renewcommand\arraystretch{1.15}
    \centering
    \small
    \begin{tabular}{c c | c c c c c c}
        \hline
        \multirow{2}{*}{\shortstack{Data \\ Model}} & \multirow{2}{*}{Defense} & \multirow{2}{*}{$\alpha$} & \multirow{2}{*}{Found} & \multirow{2}{*}{\% Removed} & \multirow{2}{*}{\shortstack{Target \\ Before}} & \multirow{2}{*}{\shortstack{Target \\ After}} & \multirow{2}{*}{\shortstack{Collateral \\ Damage}} \\
        ~ & ~ & ~ & ~ & ~ & ~ & ~ & ~ \\
        \hline
        \multirow{2}{*}{\shortstack{UTKFace \\ \sutk}} & AC & 1.0 & 100\% & 25.7\% &   0.222 & 0.188 & 3.8\% \\
         & SS & 1.0 & 45.0\% & 15\% & 0.222 & 0.290 & 1.5\% \\
        \multirow{2}{*}{\shortstack{IMDB \\ \lbert}} & AC & 2.0 & 44.7\% & 31.1\% & 0.506 & 0.477 & 1.3\%\\
        & SS & 2.0 & 13.3\% & 15\% & 0.506 & 0.517 & 0.6\% \\

    \end{tabular}
    \caption{Effects of applying different backdoor poisoning defenses to models attacked with \cmatch\ on the subpopulation with highest initial target damage for each model/dataset. AC refers to activation clustering and SS refers to spectral signatures. \% Found indicates what percentage of the poisons is correctly identified by the mitigation, while \% Removed shows the percentage of the training set actually removed. The table also lists the target damage before and after the defense is used, and the collateral damage incurred with said defense. }
    \label{tab:bdrdef}
    \vspace{-5pt}
\end{table*}

A reader skeptical of our theoretical results may wonder to what extent existing defenses for poisoning attacks can be used to protect from subpopulation attacks. We consider both defenses for availability attacks and defenses for backdoor attacks.
Defenses for availability attacks ensure that poisoning does not compromise a model's accuracy significantly.
Meanwhile, subpopulation attacks do have a modest impact on accuracy, focused on the target subpopulation.
This makes using availability attack defenses for protecting against subpopulation attacks a promising approach.
On top of this, the model classes our lower bound applies to (k-NN, overparameterized linear models) are not the same as those for which these availability attack defenses provably defend - TRIM~\cite{jagielski2018manipulating} is designed for linear regression, 
ILTM~\cite{shen2019learning} is
designed for generalized linear models (and is empirically successful for neural networks), and SEVER has provable guarantees for linear models and nonconvex models such as neural networks. Backdoor defenses, such as activation clustering~\cite{chen_detecting_2018} and spectral signatures~\cite{tran_spectral_2018} are designed to identify small anomalous subsets of training data for large neural network models, and so may be successful at identifying our subpopulation attacks.
In this section, we show that common existing defenses can fail against subpopulation attacks.

\begin{algorithm}[t]
\begin{algorithmic}
\State Input: Training data $D$ of $n$ examples, loss function $\ell$, attack count $m$, training algorithm $A$, maximum iteration count $T$
\State $\textsc{IND} =[n]$; $\textsc{IND\_PREV}=[]$; $\textsc{Iterations} = 0$
\While{$(\textsc{IND} \neq \textsc{IND\_PREV} \wedge \textsc{Iterations} < T)$ }
\State $\textsc{IND\_PREV} = \textsc{IND}$
\State $\textsc{Iterations} = \textsc{Iterations} + 1$
\State $f = A(D[\textsc{IND}])$ \Comment{Train model on good indices}
\State $\textsc{Losses} = [\ell(x_i, y_i) | (x_i, y_i)\in D]$ \Comment{Compute all losses}
\State $\textsc{IND} = \textsc{ArgSort}(\textsc{Losses})[:n-m]$ \Comment{Lowest $n-m$ indices by loss}

\EndWhile
\State \Return $f$
\end{algorithmic}
\caption{TRIM defense against availability attacks. Iteratively identifies poisoning by high loss values, and trains without those points.}
\label{alg:trim}
\end{algorithm}

We consider all five of the defenses mentioned above. When extended to generic models, TRIM and ILTM are equivalent, so we use TRIM to describe both.
TRIM/ILTM (Algorithm~\ref{alg:trim}) and SEVER both use an outlier score based on loss (for TRIM/ILTM) or gradient statistics (SEVER), and iteratively identify and remove outliers. For both, we use a maximum iteration count of $T=5$, due to the significant training time of our models, and set the attack count to be exactly the number of poisoning points, to illustrate the best possible performance of the defense.

We begin our investigation by illustrating a potential failure mode for availability attack defenses in Figure~\ref{fig:impossible_synth}; this failure mode is related to the theoretical impossibility result, but uses logistic regression, a model not captured by our theoretical results.
The figure portrays three logistic regression models - an unpoisoned model, a poisoned model, and a poisoned model with TRIM applied - trained on a synthetic binary classification dataset consisting of three subpopulations.
Each subpopulation consists of 20 data points, and 30 points are used to poison the second subpopulation.
The difference between Figures \ref{fig:unpois} and \ref{fig:pois} indicates that the subpopulation attack is fairly successful, causing a bias towards the poisoned class for that subpopulation.
This bias causes TRIM to identify the real data as the attack, rather than the poisoning, exacerbating the poisoning attack even further in Figure \ref{fig:pois_trim}. Similar results appear for the SEVER defense~\cite{diakonikolas2018sever}. These results demonstrate one failure mode of existing availability defenses at protecting against subpopulation attacks.

\begin{figure*}[t]
    \centering
\begin{subfigure}{0.33\textwidth}
  \includegraphics[width=\linewidth]{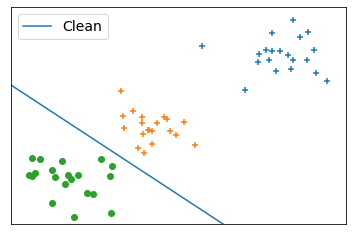}
  \caption{Unpoisoned model}
  \label{fig:unpois}
\end{subfigure}\hfil
\begin{subfigure}{0.33\textwidth}
  \includegraphics[width=\linewidth]{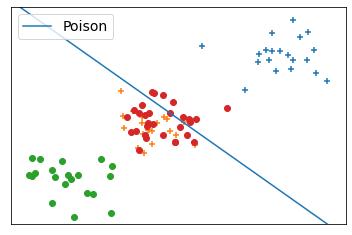}
  \caption{Poisoned model}
  \label{fig:pois}
\end{subfigure}\hfil
\begin{subfigure}{0.33\textwidth}
  \includegraphics[width=\linewidth]{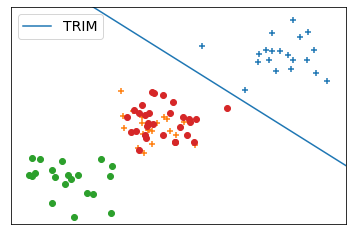}
  \caption{Poisoned + TRIM}
  \label{fig:pois_trim}
\end{subfigure}
\caption{Illustration of how standard defenses (such as TRIM~\cite{jagielski2018manipulating}/ILTM~\cite{shen2019learning}) can fail against subpopulation attacks. If the subpopulation attack can cause a bias towards the attack class, it will be exacerbated by the defense.}
\label{fig:impossible_synth}
\end{figure*}

We present the performance of availability attack defenses in Table~\ref{tab:trim}, and backdoor defenses in Table~\ref{tab:bdrdef}. We evaluate on UTKFace + \sutk\ and IMDB + \sbert. TRIM works fairly well for the two most impacted subpopulations on UTKFace, decreasing target damage from $48\%$ and $31\%$ target damage to $-2\%$ and $6\%$ target damage at $\alpha=2$. However, for many of these attacks, the target damage stays roughly the same or even increases. For example, the second most impacted subpopulation with $\alpha=.5$ has the target damage maintained at $12\%$, and the fourth most impacted subpopulation with $\alpha=2$ has target damage increase from $24\%$ to $26\%$. SEVER is similarly inconsistent. On IMDB, we run TRIM for the most damaged subpopulations at $\alpha=2$, finding no reduction in target damage on \sbert, and only a $\sim 8\%$ target damage reduction on \lbert. Backdoor defenses struggle, as well. Despite identifying all attack points on UTKFace, activation clustering removes a full 25.7\% of the training set as well, impacting performance enough that target damage is fairly unchanged. Spectral signatures has little effect on both datasets, as well.
This is evidence that more work is necessary to reliably defend against subpopulation attacks.

\section{Conclusion}
We propose subpopulation data poisoning attacks, a novel type of poisoning attack which does not require knowledge of the learner's training data or model architecture, and does not require the target testing data to be modified. We show two subpopulation selection techniques to instantiate subpopulation attacks, called \fmatch\ and \cmatch. \fmatch\ relies on manual annotations, while \cmatch\ automatically generates subpopulations based on the data, often allowing for more effective attacks. Given these subpopulation selection strategies, we also show how to generate the attack on that subpopulation---label flipping is a generic approach, but optimization-based strategies are able to improve the performance of the attack when suitable for the data modality. We provide experimental verification of the effectiveness and threat potential of subpopulation attacks using a diverse collection of datasets and models. Our datasets include tabular data, image data, and text data, and our attacks are successful on models trained both end-to-end and with transfer learning. Our subpopulation identification can also help improve existing targeted attacks. Finally, we consider defenses, proving an impossibility result which suggests some learners cannot defend against subpopulation attacks. We corroborate this by showing that existing poisoning defenses, TRIM/ILTM, SEVER, activation clustering, and spectral signatures, while sometimes successful, do not universally work at defending against our attacks, indicating future work on defenses is necessary. Our work helps understand the relationship between fair and robust machine learning, and may help understand to what extent models make decisions based on similar training data. Avenues of future work include investigating which subpopulations are at highest risk for different scenarios, and addressing the challenges of protecting ML against this novel threat.

\section*{Acknowledgments}
This research was  sponsored by the U.S. Army Combat Capabilities Development
Command Army Research Laboratory  under Cooperative Agreement Number W911NF-13-2-0045 (ARL Cyber Security CRA). The views and conclusions contained in this document are those of the authors and should not be interpreted as representing the official policies, either expressed or implied, of the Combat Capabilities Development Command Army Research Laboratory or the U.S. Government. The U.S. Government is authorized to reproduce and distribute reprints for Government purposes notwithstanding any copyright notation here on.

\bibliographystyle{alpha}
\bibliography{references}

\appendix
\section{Appendix for Defenses}
\subsection{Proof for Section~\ref{sec:defenses}}
\label{app:proof}

In this section, we will describe a theoretical model in which subpopulation attacks are \emph{impossible} to defend against.
The model is closely related to two existing theoretical models.
First is the corrupted clusterable dataset model of \cite{li2019gradient}, which was used to analyze the robustness to label noise of neural networks trained with early stopping.
The second relevant model is the subpopulation mixture model of \cite{feldman2019does} (appearing at STOC 2020), which they used to justify the importance of memorization in learning and explain the gap between differentially private learning and nonprivate learning. Our model generalizes the \cite{li2019gradient} model, and simplifies the \cite{feldman2019does} model.

The two key components of our model are the datasets, which consist of potentially noisy subpopulations of data, and the classifiers, which assign a uniform class to each cluster. \cite{li2019gradient} show that the neural network architecture and training procedure they use produces this set of classifiers if label noise is not too large. \cite{feldman2019does} shows that overparameterized linear models, k-nearest neighbors, and mixture models are examples of these classifiers, conjecturing (based on empirical evidence) that neural networks are as well.

\begin{definition}[Noisy $k$-Subpopulation Mixture Distribution]
A noisy $k$-subpopulation mixture distribution $\mathcal{D}$ over $\mathcal{X}\times \mathcal{Y}$ consists of $k$ subpopulations $\lbrace\mathcal{D}_i\rbrace_{i=1}^k$, with distinct, known, supports over $\mathcal{X}$, (unknown) mixture weights, and labels drawn from subpopulation-specific Bernoulli distributions.

We write the supports of each subpopulation $\lbrace \mathcal{X}_i\subset \mathcal{X}\rbrace_{i=1}^k$. By distinct supports, we mean that $\forall i,j\in [k]$ with $i\neq j$, $Supp_{\mathcal{X}}(\mathcal{D}_i)\cap Supp_{\mathcal{X}}(\mathcal{D}_j) = \emptyset$. Furthermore, because the supports are known, there exists a function $C_{\mathcal{D}}:\mathcal{X}\rightarrow[k]$ returning the subpopulation of the given sample.

We write the unknown mixture weights as $\boldsymbol{\alpha} = \lbrace \alpha_i\rbrace_{i=1}^k$; note that $\sum_i \alpha_i = 1$. The subpopulation-specific label distributions are written $\lbrace \text{Bernoulli}(p_i)\rbrace_{i=1}^k$. The full distribution can be written as $\mathcal{D}=\sum_i \alpha_i\mathcal{D}_i$.
\end{definition}

Notice that the existence of $C_{\mathcal{D}}$ implies that the classification task on $\mathcal{D}$ is exactly the problem of estimating the correct label for each subpopulation. We formalize this by introducing $k$-subpopulation mixture learners.

\begin{definition}[$k$-Subpopulation Mixture Learner]
A subpopulation mixture learner $\mathcal{A}$ takes as input a dataset $D$ of size $n$, and the subpopulation function $C_{\mathcal{D}}$ of a noisy $k$-subpopulation mixture distribution $\mathcal{D}$, and returns a classifier $f: [k]\rightarrow \lbrace 0, 1\rbrace$. On a fresh sample $x$, the classifier returns $f(C_{\mathcal{D}}(x))$. We call the learner \emph{locally dependent} if $f(i)$ depends only on $\lbrace y | (x, y) \in D\cap Supp_{\mathcal{X}}(\mathcal{D}_i)\rbrace $, that is, only labels from those data points which belong to the specific subpopulation. 
\end{definition}

The learners considered in both \cite{li2019gradient} (particular shallow wide neural networks) and \cite{feldman2019does} ($k$-nearest neighbors, mixture models, overparameterized linear models) are locally dependent. 

For our main theorem, we consider learners that are binary classifiers, and select a label that minimizes the cumulative error on each subpopulation using the 0-1 loss.

\begin{theorem}
For any dataset $D$ of size $n$ drawn from a  noisy $k$-subpopulation mixture distribution and subpopulation function $C_{\mathcal{D}}$, there exists a subpopulation poisoning attack $D_p$ of size $\le n/k$ that causes all locally dependent $k$-subpopulation mixture learners $\mathcal{A}$ that minimize 0-1 loss in binary classification return:

$$\mathcal{A}(D\cup D_p)=\mathcal{A}(D) \mbox{ with probability } < \frac{1}{2}. $$

Additionally, if $\alpha = \min_i \alpha_i$ is the weight of the smallest subpopulation in the mixture distribution, then a subpopulation attack of size less than $ 2\alpha n$ suffices with probability at least $ 1 - \exp(- \alpha n/2)$.
\end{theorem}

\begin{proof}
For both parts of the theorem, the poisoning attack is the same: find the subpopulation with the fewest samples in $D$, and add samples from that subpopulation with flipped labels.
No learner can distinguish between the case that the original labels are correct, and the case where the flipped labels are correct. Since the learner makes decisions based on subpopulations it will classify the subpopulation according to the majority label.
%

We address the first part of the theorem by contradiction. Assume that there exists a dataset $D$ of size $n$ drawn from a  noisy $k$-subpopulation mixture distribution, such that for any poisoning attack $D_p$ of size $\le n/k$, there exists a locally-dependent  $k$-subpopulation mixture learner $\mathcal{A}$ for which:
$$ \mathcal{A}(D\cup D_p)=\mathcal{A}(D), \mbox{ with probability } > \frac{1}{2}. $$
Assume without loss of generality that $K = \argmin_i |D\cap Supp_{\mathcal{X}}(\mathcal{D}_i)|$, and write $D_K = D\cap Supp_{\mathcal{X}}(\mathcal{D}_i)$.
Our attack operates by taking $D_K$ and flipping all of its labels, producing $D_K^p$.

Suppose we provide the learner with the original dataset $D$ and the returned classifier is $f$. On the other hand, when we provide the learner with the poisoned dataset $D' = D || D_K^p$, it returns $f'$.

Datasets $D$ and $D'$ differ in $|D_K|$ records, which is $\le n/k$, according to the pigeonhole principle ($K$ being the smallest subpopulation). 
From the properties we assumed about the learner $\mathcal{A}$ , we have:
$$ \mathcal{A}(D')=\mathcal{A}(D), \mbox{ with probability } > \frac{1}{2}, $$
which implies that learners $f$ and $f'$ return the same label for subpopulation $K$: 
$f(K) = f'(K)$.

On the other hand, the learners $f$ and $f'$ are locally dependent, and make their decisions only based on subpopulations to minimize the 0-1 loss. Because we added in $D'$ enough points in the subpopulation $K$ to flip the decision, it turns out that: $f'(K) = 1-f(K)$.

But these two statements result in a contradiction, which proves the first part of the theorem.

We now turn to the second part of the theorem statement, to improve the bounds on the size of the smallest subpopulation. This argument is particularly powerful when the mixture distribution does not have uniform weights.

For a sample from a noisy $k$-subpopulation mixture distribution, consider the smallest subpopulation $j=\argmin_i \alpha_i$ and its mixture coefficient $\alpha = \alpha_j$.

%
The number of points in subpopulation $j$ is $\sum_{i=1}^n\text{Bernoulli}(\alpha_j)$.
We use the following multiplicative Chernoff bound, with $\delta>0$:
\[
\Pr[X>(1+\delta)\mu] < \exp\left(-\frac{\delta^2\mu}{2\delta}\right).
\] Setting $\mu=\alpha_j n = \alpha n$ and $\delta=1$ in the above Chernoff bound gives

$$\Pr[|D\cap Supp_{\mathcal{X}}(\mathcal{D}_j)| > 2\alpha n] \le \exp\left(-\frac{\alpha n}{2}\right).$$

This concludes our proof, as the size of the smallest subpopulation is less than $2\alpha n$ with probability at least $1 - \exp(-\alpha n/2)$.

\end{proof}

For mixture distributions with non-uniform weights, the Chernoff bound provides more effective bounds than the pigeonhole principle. For example, if $n=1000$, the smallest mixture coefficient $\alpha=0.01$, and $k=5$, the Chernoff bound approach ensures that a poisoning attack of size 20 suffices with probability $>1-\exp(-5)=99.3\%$, whereas the pigeonhole principle requires a poisoning attack of size 200.
\begin{figure*}[h]
    \centering
\begin{subfigure}{0.33\textwidth}
  \includegraphics[width=\linewidth]{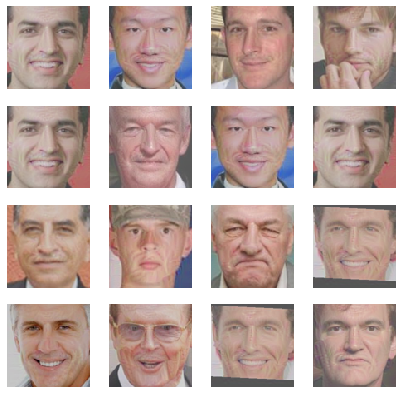}
  \caption{}
  \label{fig:cmatch1}
\end{subfigure}\hfil\vrule width 2pt \hfil
\begin{subfigure}{0.33\textwidth}
  \includegraphics[width=\linewidth]{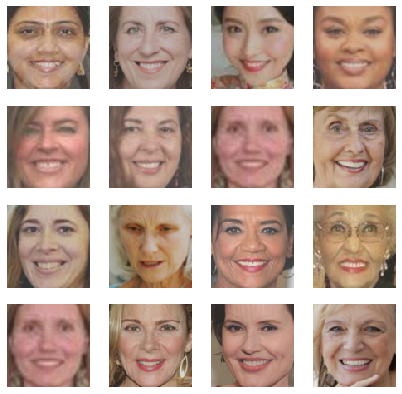}
  \caption{}
  \label{fig:cmatch2}
\end{subfigure}\hfil\vrule width 2pt \hfil
\begin{subfigure}{0.33\textwidth}
  \includegraphics[width=\linewidth]{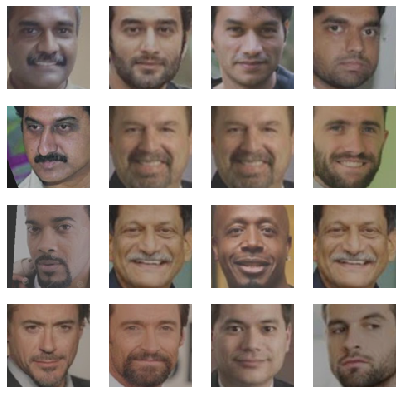}
  \caption{}
  \label{fig:cmatch3}
\end{subfigure}
\rule{\textwidth}{2pt}
\begin{subfigure}{0.33\textwidth}
  \includegraphics[width=\linewidth]{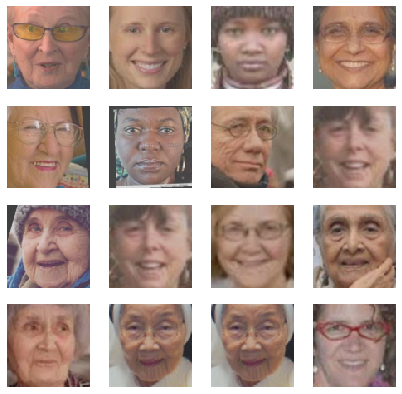}
  \caption{}
  \label{fig:cmatch4}
\end{subfigure}\hfil\vrule width 2pt \hfil
\begin{subfigure}{0.33\textwidth}
  \includegraphics[width=\linewidth]{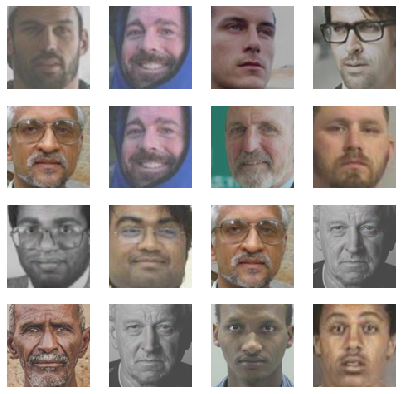}
  \caption{}
  \label{fig:cmatch5}
\end{subfigure}\hfil\vrule width 2pt \hfil
\begin{subfigure}{0.33\textwidth}
  \includegraphics[width=\linewidth]{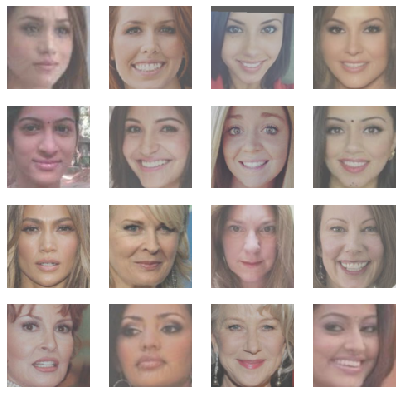}
  \caption{}
  \label{fig:cmatch6}
\end{subfigure}
\caption{Six example subpopulations generated with \cmatch\ on UTKFace. While they don't necessarily correspond perfectly to human-interpretable subpopulations, patterns do show up that may be aligned with an adversary's objective.}
\label{fig:examples}
\end{figure*}

\section{Example \cmatch\ Subpopulations}
\label{app:examples}
We present some example subpopulations in Figure~\ref{fig:examples}, generated with \cmatch\ on UTKFace. The subpopulations are not perfectly human-interpretable, but do have some consistent trends. For example, in Figure~\ref{fig:cmatch1}, the subpopulation is populated primarily by older white men, with a few exceptions. Figure~\ref{fig:cmatch2} has mostly older wthie women, and Figure~\ref{fig:cmatch3} is mostly bearded men or men with darker skin color (and potentially some white men in environments with darker lighting).

\end{document}